%% file: main.tex
\documentclass{article} 
\usepackage[preprint]{neurips_2025}
\usepackage{natbib}

\input{math_commands.tex}

\usepackage{url}

\usepackage{algorithm}
\usepackage{algorithmic} 
\usepackage[utf8]{inputenc} 
\usepackage[T1]{fontenc}    
\usepackage{hyperref}       
\usepackage{url}            
\usepackage{booktabs}       
\usepackage{amsfonts}       
\usepackage{nicefrac}       
\usepackage{microtype}      
\usepackage{xcolor}         
\usepackage{graphicx}
\usepackage{subfigure}
\usepackage{booktabs} 
\usepackage{multirow}
\usepackage{hyperref}

\usepackage{setspace}
\usepackage{cancel}
\usepackage{amsmath}
\usepackage{amssymb}
\usepackage{mathtools}
\usepackage{amsthm}
\usepackage{thm-restate}
\usepackage{booktabs}   
\usepackage{multirow}   
\usepackage{tabularx}   

\usepackage[capitalize]{cleveref}
\crefname{appendix}{App.}{Apps.}
\Crefname{appendix}{App.}{Apps.}

\crefname{corollary}{Cor.}{Cors.}
\Crefname{corollary}{Cor.}{Cors.}
\crefname{section}{Sec.}{Secs.}
\Crefname{section}{Sec.}{Secs.}
\crefname{proposition}{Prop.}{Props.}
\Crefname{proposition}{Prop.}{Props.}

\theoremstyle{plain}
\newtheorem{theorem}{Theorem}[section]
\newtheorem{proposition}[theorem]{Proposition}
\newtheorem{lemma}[theorem]{Lemma}

\theoremstyle{definition}

\theoremstyle{remark}
\newtheorem{remark}{Remark}

\usepackage[textsize=tiny]{todonotes}
\usepackage[normalem]{ulem}

\usepackage{pifont} 
\usepackage{minted}      

\newcommand{\cmark}{\ding{51}} 
\newcommand{\xmark}{\ding{55}} 

\title{Measurement-Aligned Sampling for \\
Inverse Problems}

\author{Shaorong Zhang$^{1}$, \quad Rob Brekelmans $^{2}$, \quad Yunshu Wu$^{1}$, \quad Greg Ver Steeg$^{1}$ \\
$^1$University of California Riverside \quad $^2$ Vector Institute \\
\{szhan311, ywu380, gregoryv\}@ucr.edu,  \quad rob.brekelmans@vectorinstitute.ai
}

\begin{document}

\maketitle

\begin{abstract}
Diffusion models provide a powerful way to incorporate complex prior information for solving inverse problems. However, existing methods struggle to correctly incorporate guidance from conflicting signals in the prior and measurement, and often failed to maximizing the consistency to the measurement, especially in the challenging setting of non-Gaussian or unknown noise. To address these issues, we propose Measurement-Aligned Sampling (MAS), a novel framework for linear inverse problem solving that flexibly balances prior and measurement information. MAS unifies and extends existing approaches such as DDNM, TMPD, while generalizing to handle both known Gaussian noise and unknown or non-Gaussian noise types. Extensive experiments demonstrate that MAS consistently outperforms state-of-the-art methods across a variety of tasks, while maintaining relatively low computational cost. 
\end{abstract}

\section{Introduction}


Inverse problems are prevalent in image restoration (IR) tasks, including super-resolution, inpainting, deblurring, colorization, denoising, and JPEG restoration \citep{chung2022diffusion, kawar2022denoising, saharia2022image, wang2022zero, lugmayr2022repaint, mardani2023variational, song2023pseudoinverse, kawar2022jpeg}. Solving an inverse problem involves recovering an unknown original image $x_0 \in \mathbb{R}^n$ based on information from a prior distribution, $\pi(x_0)$, and noisy measurements $y \in \mathbb{R}^m$ generated through a forward model:

\begin{equation}
    y = \mathcal{H}(x_0) + \epsilon.
\end{equation}

Here ${\epsilon \in \mathbb{R}^m}$ represents measurement noise, $x_0 \in \mathbb{R}^d$ is drawn from data distribution $\pi_0(x_0)$, $\mathcal{H}:\mathbb{R}^d\mapsto\mathbb{R}^m$ is the measurement function, and $y \in \mathbb{R}^m$ denotes the degraded measurement or observed image. A useful motivating example is a high-resolution image $x_0$, with a noisy degraded image $y$ and a known corruption process.


Pretrained diffusion and flow models offer a prior distribution $\pi_0(x_0)$ that greatly aids in solving inverse problems. Methods such as DPS \citep{chung2022diffusion}, $\Pi$GDM \citep{song2023pseudoinverse}, and TMPD \citep{boys2023tweedie} estimate conditional scores directly from the measurement model by leveraging score decomposition to guide each diffusion sampling step. In contrast, approaches like FPS \citep{dou2024diffusion}, DAPS \citep{zhang2024improving}, MPGD \citep{he2023manifold}, and optimization-based methods \citep{song2023solving,zhu2023denoising,li2024decoupled,wang2024dmplug} align denoiser outputs directly with measurements, thereby avoiding backpropagation through the U-Net. Although DAPS achieves state-of-the-art performance—outperforming methods that require backpropagation—it still requires more than 100 gradient descent iterations per diffusion step, making it far more computationally expensive compared to methods such as DDNM \citep{wang2022zero} and DDRM \citep{kawar2022denoising}. This highlights the promise of developing approaches that avoid both backpropagation through U-Net and excessive optimization steps, while still attaining state-of-the-art performance.


Moreover, the above approaches lack the ability to effectively handle unknown or non-Gaussian noise. In practical settings, noise frequently deviates from Gaussian assumptions—exhibiting characteristics like salt-and-pepper, periodic, or Poisson distributions—or is completely unknown. Additionally, the forward measurement operator may be uncertain or inaccurately specified. Effectively addressing inverse problems under these more general and realistic conditions remains an open and challenging research area. 



Our main contributions are summarized as follows:

\begin{itemize}
\item We propose \textit{Measurement-Aligned Sampling} (MAS), a novel framework for solving linear inverse problems. MAS provides both probabilistic and optimization perspectives and generalizes approaches such as DDNM and TMPD for linear inverse problems. Furthermore, our proposed `overshooting' technique achieves superior restoration quality compared to DDNM across various inverse problem scenarios.
\item We develop new techniques that \textit{maximize consistency with the measurement}, enabling robust handling of both Gaussian noise and unknown noise sources. Moreover, our novel parameterization scheme allows us to effectively handle noisy inverse problems with unknown or non-Gaussian noise structures and even non-differentiable measurements, such as JPEG restoration, without requiring explicit knowledge of the forward operator or noise level. The comparison of method applicability across different inverse problems is shown in \cref{tab:method_app}.
\item Our experiments show that MAS enables robust and efficient image restoration, consistently outperforming baselines across Gaussian, non-Gaussian, and non-differentiable degradations (see \cref{fig:restoration} and experiments in \cref{sec:experiments}).
\end{itemize}

\begin{figure}[t]
    \centering
    
\includegraphics[width=1.\linewidth]{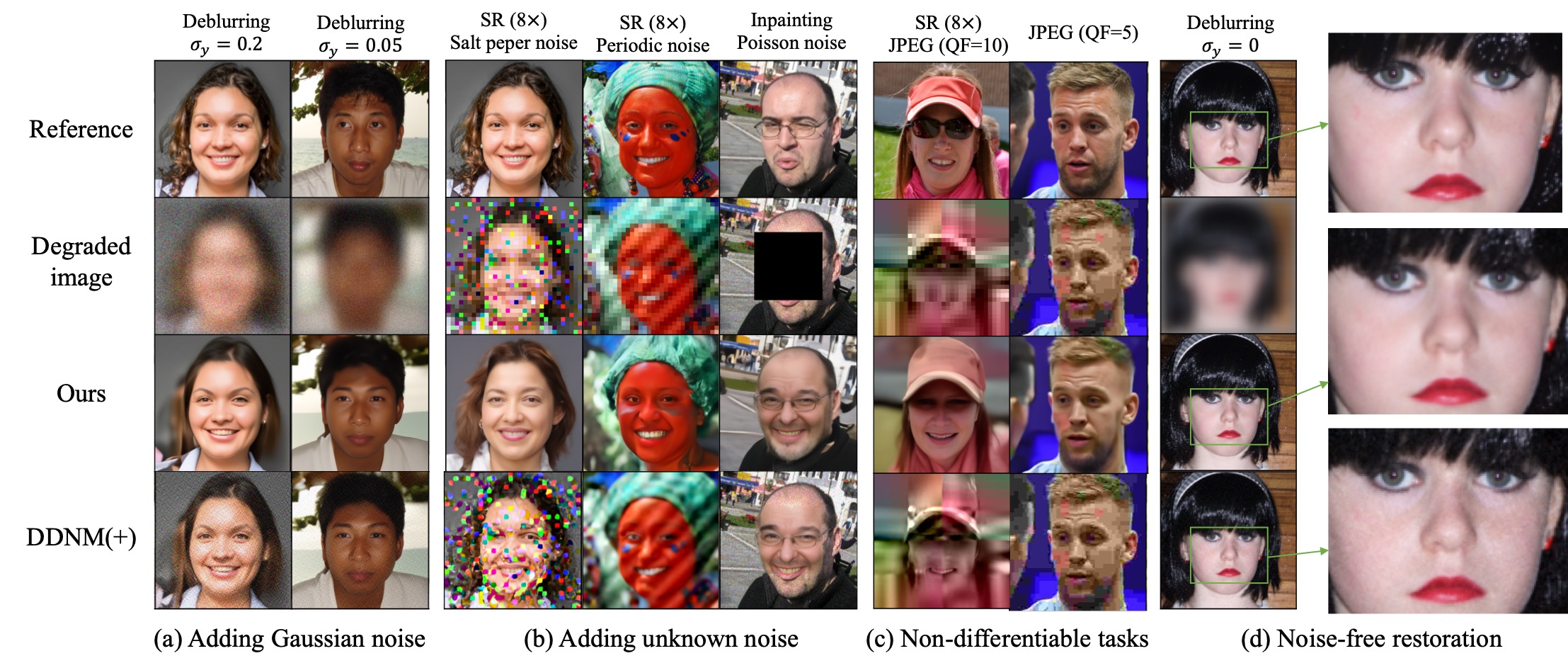}
    \caption{Solving various inverse problems using unconditional diffusion models. Our model demonstrates better robustness with unknown noise and strong Gaussian noise.}
    
\label{fig:restoration}
\end{figure}

\begin{table}[t]
\caption{Comparison of method applicability across different inverse problems.}
\label{tab:method_app}
    \centering
    \resizebox{1.\textwidth}{!}{
    \begin{tabular}{cccccccc}
    \toprule
    Inverse Problem & Noise strength  & DDNM &  DDRM &  $\Pi$GDM & DAPS & RED-Diff & MAS (ours) \\
     \midrule
    Linear + noise free & - & \cmark & \cmark & \cmark & \cmark & \cmark & \cmark \\
    Linear + Gaussian noise &  Known & \cmark & \cmark & \cmark & \cmark & \cmark & \cmark  \\
    Linear + non-Gaussian noise &  Unknown & \xmark & \xmark & \xmark & \xmark & \xmark & \cmark  \\
    \midrule
    JPEG / Quantization restoration & Known & \xmark & \xmark & \cmark & \xmark & \xmark & \cmark \\
    JPEG / Quantization restoration & Unknown & \xmark & \xmark & \xmark & \xmark & \xmark & \cmark \\
    \bottomrule
    \end{tabular}}
    
\end{table}

\section{Background}


Given training dataset $\mathcal{D}=\{ x_0^{i} \}_{i=1}^N$ from target distribution $\pi_0(x_0)$, $x_0^{i} \in \mathbb{R}^d$, the goal of generative modeling is to draw new samples from $\pi_0$. In the context of conditional generation, suppose that we have data samples from a joint distribution $(x_0^{i},y) \sim \pi(x_0, y)$, where $x_0$, $y$ are dependent, and $y$ could be class labels or text information, for example.

For conditional generative modeling, we seek to draw new 
 samples from $\pi(x_0 \mid y)$ for a given condition $y$. Conditioned flows \citep{zheng2023guided} build a 
marginal probability
path $p_{t \mid y}$ 
using a mixture of interpolating densities: 
$p_{t \mid y}(x_t \mid y) = \int p_t(x_t \mid x_0) \pi( x_0 \mid y) dx_T$, where $p_t(\cdot \mid x_0)$ is a probability path interpolating between noise and a single data point $x_T$. In general, the conditional kernel $p_t(x_t \mid x_0)$ is given by a Gaussian distribution: $p_t(x_t \mid x_0) = \mathcal{N} (x_t; \alpha_t x_0, \sigma_t^2 \mathbb{I})$, where $\mathcal{N}$ is the Gaussian kernel, $\alpha_t, \sigma_t$ are differentiable functions. Then we can sample from the conditional distribution $p_{0 \mid y}(x_0 \mid y)$ by simulating a stochastic process $p_{t \mid y} (x_t \mid y)$ from time $t = T$ to $t = 0$. Although different sampling methods can be chosen, generally, the iteration follows the form:




\begin{equation}
\label{eq:reverse_iter}
    x_{t - \Delta t} \sim \mathcal{N} (a_t m_{0 \mid t, y} + b_t x_t, c_t^2 \mathbb{I}).
\end{equation}

where $m_{0 \mid t, y} = \mathbb{E}[x_0 \mid x_t, y]$ is the idea conditional denoiser, $a_t$, $b_t$ and $c_t$ are parameters that depends on samplers. For instance, $x_{t-\Delta t}  \sim \mathcal{N} (\alpha_{t-\Delta t} m_{0 \mid t, y}, \sigma_{t-\Delta t} \mathbb{I})$ is a valid DDIM sampler. In the implementation of conditional diffusion models, a denoiser is trained to approximate $m_{0 \mid t, y}$. However, when only an unconditional denoiser $m_{0 \mid t} = \mathbb{E}[x_0 \mid x_t]$ is available, training-free conditional inference methods are employed.


\textbf{Diffusion Posterior Sampling (DPS) and its variants}. Given unconditional denoiser $\mathbb{E}[x_0 \mid x_t]$, training-free conditional inference methods enable the approximation of the ideal conditional denoiser $\mathbb{E}[x_0 \mid x_t, y]$ \citep{pokle2023training}:
\begin{equation}
\label{eq:denoiser_decomp}
     \mathbb{E}[x_0 \mid x_t, y] =  \mathbb{E}[x_0 \mid x_t] + \frac{\sigma_t^2}{\alpha_t} \nabla_{x_t} \log p(y \mid x_t).
\end{equation}

Since $\nabla_{x_t} \log p(y \mid x_t)$ is generally intractable, various approaches have been developed to approximate it. 

For linear inverse problems, where the forward model is given by: $y = Hx_0 + \epsilon$, $\epsilon \sim \mathcal{N}(0, \sigma_y^2 \mathbb{I})$. Tweedie Moment Projected Diffusion (TMPD) \citep{boys2023tweedie} assume $p(x_0 \mid x_t)$ as a Gaussian: $p(x_0 \mid x_t) \approx \mathcal{N} (m_{0 \mid t}, C_{0 \mid t})$, where $m_{0 \mid t} (x_t) := \mathbb{E}[x_0 \mid x_t]$ is the ideal unconditional denoiser, $C_{0 \mid t} (x_t) := \mathbb{E}[(x_0 - m_{0\mid t})(x_0 - m_{0 \mid t})^T \mid x_t]$ is the covariance of $x_0 \mid x_t$. Then the posteroir mean $\mathbb{E}[x_0 \mid x_t, y]$ admits an explicit closed-form solution:

\begin{equation}
\label{eq:tmpd}
    \mathbb{E}[x_0 \mid x_t, y] = m_{0 \mid t} + C_{0 \mid t} H^T (HC_{0 \mid t}H^T + \sigma_y^2 \mathbb{I})^{-1} (y - Hm_{0 \mid t})
\end{equation}

The covariance $C_{0 \mid t}$ could be calculated via gradient go through the denoiser: $C_{0 \mid t} = \frac{\sigma_t^2}{\alpha_t} \nabla_{m_{0 \mid t}} (x_t)$.

\textbf{Optimization based methods}. Unlike DPS guarantees that sampling is strictly from the conditional distribution, $p(x_0 \mid y)$, optimization-based approaches \citep{zhu2023denoising, li2024decoupled, wang2024dmplug} place more emphasis on the alignment with the measurement and the prior, which takes the following iteration:

\begin{subequations}
\label{eq:opt_iter}
\begin{align}
\label{eq:objective}
    x_0^* &= \arg \min_{x_0} \; \bigl\| x_0 - m_{0 \mid t} \bigr\|^2 
    + \lambda_t \, \bigl\| y - \mathcal{H}(x_0) \bigr\|^2, \\
    x_{t-\Delta t} &\sim \mathcal{N}\!\left( a_t x_0^* + b_t x_t, \, c_t^2 \right).
\end{align}
\end{subequations}

where $\lambda_t$ is a manually designed hyperparameter and $\mathcal{H}(\cdot)$ is the nonlinear forward operator. The iteration of optimization based methods could be seen as replacing $m_{0 \mid t, y}$ in \cref{eq:reverse_iter} to $x_0^*$ in \cref{eq:objective}.

\section{Methodology}\label{sec:method}



For optimization based methods, the data-consistency loss with respect to the measurement $y$ is treated uniformly across all directions of the measurement space. However, for inverse problems it is often advantageous to introduce a weighting matrix that reflects the geometry of the forward operator \citep{tarantola2005inverse}. To this end, we propose Measurement-Aligned Sampling (MAS), which incorporates such a weighting into the optimization. As we demonstrate in \cref{sec:experiments}, this alignment leads to significant improvements in reconstruction quality.

\subsection{Measurement Aligned Sampling}

In this work, we generalize the objective in \cref{eq:objective} as

\begin{equation}
\begin{aligned}
\label{eq:opt_per}
    x_0^* 
    &= \arg \min_{x_0} \Vert x_0 - m_{0 \mid t} \Vert^2 +\Vert y - Hx_0 \Vert_{W^{-1}}^2. 
\end{aligned}
\end{equation}

where $W \coloneqq \eta_1 HH^{\mathsf T} + \eta_2 \mathbb{I}$ (with $\eta_1\geq0, \eta_2 \geq 0$) is the weighted matrix and serves as a metric that balances measurement fidelity and prior regularization, where $\| z \|^2_{A} = z^T A z$.

When $\eta_1 = 0$ and $\eta_2 > 0$, corresponding to the classical Tikhonov (ridge) regularization, where $\eta_2$ controls the trade-off between fitting the measurements $y$ and staying close to the prior $m_{0\mid t}$. When $\eta_1 > 0$ and $\eta_2 = 0$, the data term becomes weighted by $(HH^{\top})^{-1}$, a Mahalanobis-type distance that emphasizes alignment along directions where $H$ is weak (small singular values), thereby regularizing ill-posed components of the inverse problem. \emph{The balance between $\eta_1$ and $\eta_2$ plays a crucial role in reconstruction quality}, as we show in our experiments.

Finally, \cref{eq:opt_per} admits a unique closed-form solution obtained by setting the gradient to zero:



\begin{equation}
\begin{aligned}
    \label{eq:pm}
x_0^*  &= Y^{-1}[m_{0 \mid t} + H^{\mathsf T} W^{-1}y] \\
\text{where} \quad 
W & \coloneqq \eta_1 HH^{\mathsf T} + \eta_2 \mathbb{I}, \qquad Y \coloneqq \mathbb{I} + H^{\mathsf T} W^{-1}H,
\end{aligned}
\end{equation}

In practice, computing the inverse $W^{-1}$ and  $Y^{-1}$ in \cref{eq:pm} naively can be computationally expensive. Instead, one can employ singular value decomposition (SVD) for more efficient computation; see \cref{sec:eff_cal} for details.

    

\begin{remark}
[\textbf{Connection with DDNM} \citep{wang2022zero}]
\label{remark:con_ddnm}
As $\eta_2  = 0$ and $\eta_1 \rightarrow 0$, $x_0^* \rightarrow \tilde{x}_0^{\text{DDNM}} := m_{0 \mid t} + H^{\dagger} (y - H m_{0 \mid t})$. Thus, in this limiting case, MAS recovers DDNM. 
\end{remark}

\begin{remark}[\textbf{Connection with optimization methods}]
For the case where $\eta_1 = 0, \eta_2 >0$,
\cref{eq:opt_per} reproduces optimization approaches, such as Resample \citep{song2023solving}, DiffPIR \citep{zhu2023denoising}, DCDP \citep{li2024decoupled}, DMPlug \citep{wang2024dmplug}.
\end{remark}




\subsection{Probabilistic interpretation}\label{sec:shallow_guidance}

We can interpret $x_0^*$ in \cref{eq:pm} as $\mathbb{E}[x_{\epsilon} \mid x_t, y]$, where $x_{\epsilon} \approx x_0$ with perturbation variance $\sigma_{\epsilon}^2$ chosen to be sufficiently small so that $p(x_0 \mid x_\epsilon) \approx \mathcal{N} (x_\epsilon, \sigma_\epsilon^2 \mathbb{I})$. This formulation introduces an additional hyperprior, which—as demonstrated in our experiments—proves beneficial in addressing inverse problems. 

\begin{algorithm}[t]

\caption{\small Measurement-Aligned Sampling (MAS) for inverse problems.}
\label{algo:mas}
\begin{algorithmic}[1]
\small
\STATE \textbf{Input:} measurement $y$, forward operator $H (\cdot)$, pretrained DM $\epsilon_{\theta} (\cdot)$, number of diffusion step $N$, diffusion schedule $\alpha_t$ and $\sigma_t$, objective parameters $\eta_1$, $\eta_2$.
\STATE \textbf{Initialization}: $x_N \sim \mathcal{N} (0, \mathbb{I})$
\FOR{$n=N$ to $1$}
    \STATE $\hat{x}_0 \leftarrow [{x_n - \sigma_n \epsilon_{\theta} (x_n, n)}]/{\alpha_n}$ \hfill $\triangleright$ Obtain predicted data $\mathbb{E}[x_0 \mid x_n]$
    \STATE $x_0' = Y^{-1} [\hat{x}_0 + H^{\mathsf T}W^{-1}y]$ \hfill $\triangleright$ Calculating posterior mean $\mathbb{E}[x_\epsilon \mid x_n, y]$
    \STATE $x_{n-1} \sim \mathcal{N} ( \alpha_{n-1} x_0', \sigma_{n-1} \mathbb{I})$ \hfill $\triangleright$ Forward diffusion step
\ENDFOR
\STATE \textbf{Output} $x_0$
\end{algorithmic}
\end{algorithm}

Given the measurement model
$p(y \mid x_0) = \mathcal{N}(H x_0, \sigma_y^2 \mathbb{I})$ and conditional $  
p(x_0 \mid x_{\epsilon}) = \mathcal{N}(x_{\epsilon}, \sigma_{\epsilon}^2 \mathbb{I})$, the induced distribution over the measurement conditioned on $x_{\epsilon}$ takes the explicit form:
\begin{equation}
    p(y \mid x_{\epsilon}) = \mathcal{N}\!\left(Hx_{\epsilon},\, \sigma_y^2 \mathbb{I} + \sigma_{\epsilon}^2 HH^{\mathsf T}\right).
\end{equation}

Notably, the likelihood $p(y \mid x_{\epsilon})$ shares the same mean as $p(y \mid x_0)$, but with a generalized variance inflated by a term depending on $H$. Since both $p(y \mid x_{\epsilon})$ and $p(y \mid x_0)$ are Gaussian, the posterior distribution admits a closed-form expression. In particular, the posterior mean $\mathbb{E}[x_0 \mid x_t, y]$ can be computed via Bayesian linear regression, as stated in \cref{prop:blr}.


\begin{restatable}[Bayesian Linear Regression]{proposition}{propgaussian}
\label{prop:blr}
Suppose $p(y \mid x_{\epsilon}) = \mathcal{N}\!\left(Hx_{\epsilon},\,R\right)$, $R \coloneqq \sigma_y^2 \mathbb{I} + \sigma_{\epsilon}^2 HH^{\mathsf T}$ 
and $p(x_{\epsilon} \mid x_t) \approx \mathcal{N}\!\left(m_{0 \mid t},\, C_{0 \mid t}\right)$. Then the posterior is Gaussian, with mean given by
\begin{equation}
\label{eq:blr}
\mathbb{E}[x_{\epsilon} \mid x_t, y]
= \Big(C_{0 \mid t}^{-1} + H^{\mathsf T} R^{-1} H \Big)^{-1}
   \Big(C_{0 \mid t}^{-1} m_{0 \mid t} + H^{\mathsf T} R^{-1} y \Big).
\end{equation}
\end{restatable}

As we set $C_{0 \mid t} = r_t^2 \mathbb{I}$, $\eta_1 \coloneqq {\sigma_\epsilon^2}/{r_t^2}$ and $\eta_2 \coloneqq \sigma_y^2 / r_t^2$, $\mathbb{E}[x_\epsilon \mid x_t, y]$ in \cref{eq:blr} is equivalent to $x_0^*$ in \cref{eq:pm}, which provides a probabilistic perspective for MAS.

\begin{remark}[\textbf{Connection to TMPD \citep{boys2023tweedie}}]
Setting $\sigma_\epsilon=0$ reduces the posterior mean in \cref{eq:blr} to that of TMPD in \cref{eq:tmpd}.
\end{remark}

\begin{remark}[\textbf{`Overshooting' trick}]
    Theoretically, $\eta_1 \geq 0$ since $\eta_1\coloneqq \sigma_{\epsilon}^2 / r_t^2$, however, the posterior mean in \cref{eq:pm} allows negative $\eta_1$. As illustrated in \cref{fig:toy}, negative $\eta_1$ produces an overshooting effect, drawing $x_0^*$ even further toward alignment with the measurement $y$ than prescribed by DDNM. Interestingly, in our experiments this overshooting effect leads to improved reconstruction quality. A more detailed discussion is provided in \cref{sec:eff_cal}.
\end{remark}

\begin{figure}
    \centering
    \includegraphics[width=1.0\linewidth]{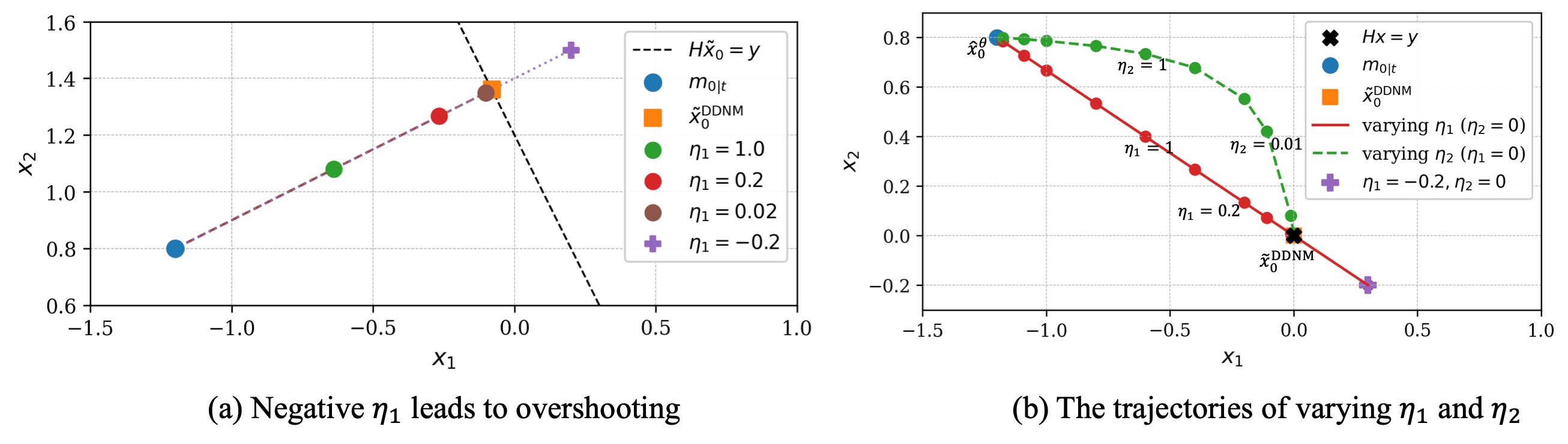}
    \caption{2D illustration of the influence of parameters $\eta_1$ and $\eta_2$. Dots represent $x_0^*$, calculated via \cref{eq:pm}. (a) Parameter $\eta_1$ controls the trade-off between $m_{0 \mid t}$ and $\tilde{x}_0^{\text{DDNM}}$: as $\eta_1 \rightarrow \infty$, the posterior mean $x_0^*$ approaches $m_{0 \mid t}$; as $\eta_1 \rightarrow 0$, it converges to $\tilde{x}_0^{\text{DDNM}}$. (b) Adjusting $\eta_2$ alters the posterior trajectory differently from varying $\eta_1$.}
    \label{fig:toy}
\end{figure}


\section{Maximizing the consistency for noisy inverse problems}\label{sec:advanced}



\subsection{Why previous methods failed to maximize the consistency?}






DDNM highlighted that calculating the posterior sampling $\tilde{x}_0^{\text{DDNM}} = m_{0 \mid t} + H^{\dagger} (y - H m_{0 \mid t})$ can inadvertently introduce additional noise into $x_t$, if $y$ is noisy. For instance, consider a simple forward model: $y=x_0 + \epsilon_y$, where both $H$ and $H^{\dagger}$ are identity matrix, i.e., $H = H^{\dagger} = \mathbb{I}$, then $\tilde{x}_0^{\text{DDNM}} = y = x_0 + \textcolor{blue}{\epsilon_y}$. Here $\epsilon_y$ is the additional noise introduced to $\tilde{x}_0^{\text{DDNM}}$, and will be further introduced into $x_{t-\Delta t}$. We argue that this issue is not unique to DDNM, but may also arise in TMPD \citep{boys2023tweedie}, DAPS \citep{zhang2024improving}, as well as in optimization-based methods \citep{zhu2023denoising, li2024decoupled}.

For MAS and under this same example, $y=x_0 + \epsilon_y$, calculating $x_0^*$ (\cref{eq:pm}) yields


\begin{equation}
\label{eq:pm_simple}
    x_0^*
  = m_{0 \mid t} +   \frac{y - m_{0 \mid t}}{\eta_1 + \eta_2 + 1} = m_{0 \mid t} +   \frac{x_0 - m_{0 \mid t}}{\eta_1 + \eta_2 + 1} + \textcolor{blue}{\frac{\epsilon_y}{\eta_1 + \eta_2 + 1}},
\end{equation}

where \textcolor{blue}{$\epsilon_y / (\eta_1 + \eta_2 + 1)$} is the additional noise introduced to $x_0^*$. A delicate balance arises from the fact that increasing either $\eta_1$ or $\eta_2$ will not only reduce the influence of the (unknown) noise term $\epsilon_y$, but also reduce the consistency with the measurement $y$ in general. To address this issue, we propose two approaches for addressing known Gaussian noise (\cref{sec:known_gaussian}) and unknown noise \cref{sec:unknown_noise}.

\subsection{Addressing Gaussian noise with known variance}
\label{sec:known_gaussian}



To handle Gaussian noise with known variance and $H= \mathbb{I}$, we modify \cref{eq:pm_simple} and \cref{eq:reverse_iter} as:

\begin{equation}
    x_0^* = m_{0 \mid t} + \textcolor{blue}{\lambda_t}  \frac{y - m_{0 \mid t}}{\eta_1 + \eta_2 + 1},  \quad
    x_{t - \Delta t} \sim \mathcal{N} (a_t \tilde{x}_0 + b_t x_t, \textcolor{blue}{\gamma_t} \mathbb{I}).
\end{equation}

Here $\lambda_t$ and $\gamma_t$ are two parameters that can control the total noise introduced to $x_{t-\Delta t}$. In our work, we adopt similar two principles as DDNM$+$ \citep{wang2022zero} for handling Gaussian noise: (i) the total noise introduced in $x_{t -\Delta t}$ should be $\mathcal{N} (0, c_t^2 \mathbb{I})$ to conform to the correct distribution of $x_{t-\Delta t}$ in \cref{eq:reverse_iter}; (ii) $\lambda_t$ should be as close to $1$ as possible to maximize the preservation of $x_0^*$. As $\epsilon_y \sim \mathcal{N} (0, \sigma_y^2 \mathbb{I})$, principle (i) and principle (ii) are equivalent to: 
\begin{equation}
\begin{aligned}
    \left( \frac{a_t \lambda_t \sigma_y}{\eta_1 + \eta_2 + 1} \right)^2 + \gamma_t &= c_t^2, \quad 
    \lambda_t = 
    \begin{cases}
        1, & c_t \geq \dfrac{a_t \sigma_y}{\eta_1 + \eta_2 + 1} \\[4pt]
        \dfrac{c_t (\eta_1 + \eta_2 + 1)}{a_t \sigma_y}, & c_t < \dfrac{a_t \sigma_y}{\eta_1 + \eta_2 + 1}
    \end{cases}.
\end{aligned}
\end{equation}

Derivations for more general forms of $H$ can be found in \cref{sec:sm}.
Note that the revision does not introduce additional parameters. 

\subsection{Addressing unknown noise and non-differentiable measurements}
\label{sec:unknown_noise}

\begin{figure}
    \centering
\includegraphics[width=1.\linewidth]{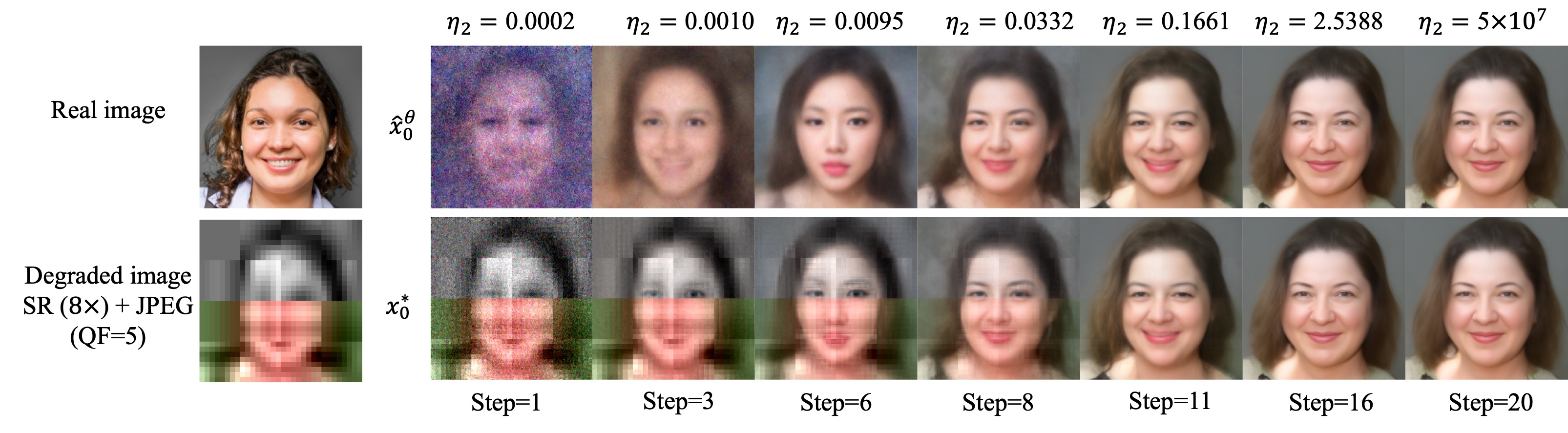}
    \caption{The sample process of solving inverse problems with unknown noise, where $\hat{x}_0^{\theta}\approx m_{0 \mid t}$ is the denoising output. Here we set $\eta_1 = 0$ and $\eta_2 = 0.5 a_t/c_t$.}
    \label{fig:sample_process}
\end{figure}

\textbf{Addressing unknown noise or non-Gaussian noise.} When the measurement noise $\sigma_y$ is non-Gaussian or unknown, it becomes difficult to ensure that the total noise in $x_{t-\Delta t}$ follows the desired distribution $\mathcal{N} (0, c_t^2 \mathbb{I})$, To address this, we continue to sample $x_{t-\Delta t}$ using \cref{eq:reverse_iter}. Next, the noise introduced to $x_{t-\Delta t}$ is the sum of two components:

\begin{equation}
    \epsilon_{\text{ng}} = (a_t \lambda_t \epsilon_y) / (\eta_1 + \eta_2 + 1), \quad \epsilon_{\text{g}} \sim \mathcal{N} (0, c_t^2 \mathbb{I}).
\end{equation}

Here $\epsilon_{\text{ng}}$ is related to the noise introduced by unknown noise $\epsilon_y$, while $\epsilon_{\text{g}}$ is the noise added by the diffusion process. To minimize the effect of unknown noise $\epsilon_{\text{ng}}$, it is desirable for $\eta_1 + \eta_2 + 1$ to be sufficiently large. However, smaller values of $\eta_1$ and $\eta_2$ result in better consistency with the measurement $y$, as illustrated in \cref{fig:toy}. To balance this trade-off, we propose using a small $\eta_1 + \eta_2$ during the early stages of sampling to fully exploit measurement information. As sampling progresses, $\eta_1 + \eta_2$ should be gradually increased to suppress the impact of $\epsilon_{\text{ng}}$. The underlying intuition is that, in the early sampling stage, $x_t$ is still highly noisy and $a_t \approx 0$, so the influence of $\epsilon_{\text{ng}}$ is negligible even when $\eta_1 + \eta_2$ is small. As shown in \cref{fig:sample_process}, $x_0^*$ is initially more aligned with the degraded observation, but progressively shifts toward $m_{0 \mid t}$ as sampling evolves. 

For a general degradation operator $H$, we recommend setting $\eta_2 = k a_t / c_t$, where $k$ is a constant determined by the characteristics of the introduced noise. The rationale behind this design choice is further detailed in \cref{sec:sm}.






\textbf{Addressing non-differentiable measurements.} For solving inverse problems with non-differentiable measurements such as JPEG restoration and quantization, the degraded images can be viewed as "noisy images" with unknown noise, modeled by $y = x + \epsilon_y$. In these scenarios, our proposed strategy naturally extends by treating the unknown degradations as implicit noise.

\section{Experiments}
\label{sec:experiments}

\subsection{Experimental setup}
\textbf{Dataset}. We evaluate the effectiveness of our proposed approach on FFHQ $256 \times 256$ \citep{karras2019style} and ImageNet $256 \times 256$ \citep{deng2009imagenet}. Following DAPS \citep{zhang2024improving}, we test on the same subset of 100 images for both datasets.

\textbf{Pretrained models and baselines}. We utilize the pre-trained checkpoint \citep{chung2022diffusion} on the FFHQ dataset and the pre-trained checkpoint \citep{dhariwal2021diffusion} on the Imagenet dataset. We compare our methods with the following baselines: DCDP \citep{li2024decoupled}, FPS \citep{dou2024diffusion}, DiffPIR \citep{zhu2023denoising}, DDNM \citep{wang2022zero}, DDRM \citep{kawar2022denoising}, $\Pi$GDM \citep{song2023pseudoinverse}, RedDiff \citep{mardani2023variational},  DAPS \citep{zhang2024improving}.

\textbf{Metrics}. Following previous work \citep{chung2022diffusion, kawar2022denoising}, we report Fréchet Inception Distance (FID) \citep{heusel2017gans}, Learned Perceptual Image Patch Similarity (LPIPS) \citep{zhang2018unreasonable}, Peak Signal-to-Noise Ratio (PSNR), and Structural SIMilarity index (SSIM). 

\textbf{Tasks}. (1) We evaluate performance on the following linear inverse problems: super-resolution (bicubic filter), deblurring (uniform kernel of size $9$), inpainting (with a box mask), inpainting (with a 70\% random mask), and colorization.
(2) We consider two unknown noise types: salt-and-pepper noise ($10 \%$ pixels set randomly to $\pm 1$) and periodic noise (sinusoidal pattern with amplitude $0.2$ and frequency $5$).
(3) We address JPEG restoration with quality factors $\mathrm{QF}=2$ and $\mathrm{QF}=5$. (4). For quantization, we consider the challenging case of 2-bit quantization.

\subsection{Ablation study}

\begin{figure}[t]
    \centering
\includegraphics[width=1.\linewidth]{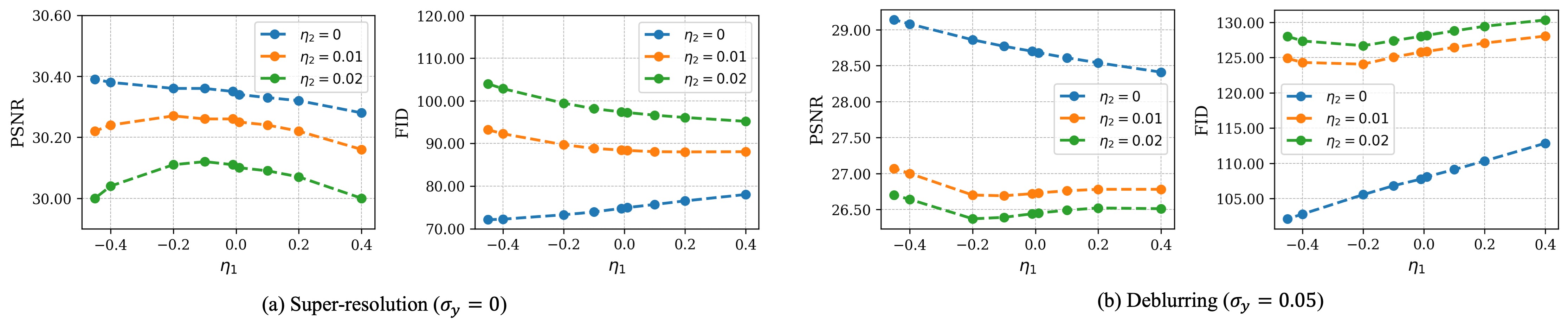}
\caption{Ablation study of $\eta_1$ and $\eta_2$ on solving super-resolution and deblurring. We set NFE=20 for all tasks.}
    \label{fig:ablation}
\end{figure}

\paragraph{Ablation Study on $\eta_1$ and $\eta_2$.}
We conduct ablation studies on parameters $\eta_1$ and $\eta_2$ using two inverse problems: super-resolution (noise-free, $\epsilon_y = 0$) and deblurring (noisy, $\epsilon_y \sim \mathcal{N}(0, \sigma_y^2 \mathbb{I})$). Results presented in \cref{fig:ablation} demonstrate that for noise-free super-resolution, the highest PSNR and lowest FID scores are achieved by setting $\eta_2 = 0$ and a negative $\eta_1 = -0.45$. This indicates that appropriate "overshooting" enhances restoration quality. For the noisy deblurring task, negative $\eta_2$ yields an improvement of more than 0.5 in PSNR and a reduction of over 5 in FID, further confirming the benefit of overshooting. 


\begin{table*}[!t]
  \centering
  \caption{Quantitative evaluation of solving image restoration FFHQ (left) and ImageNet (right), with Gaussian noise (known variance, $\sigma_y = 0.05$).}
  \label{tab:noisy_gaussian}
  \resizebox{0.95 \textwidth}{!}{%
  \begin{tabular}{@{}llcccccccc@{}}
  \toprule
        \multirow{2}{*}{\textbf{Task}}  & \multirow{2}{*}{\textbf{Method}}
      & \multicolumn{4}{c}{\textbf{FFHQ}} & \multicolumn{4}{c}{\textbf{ImageNet}} \\ \cmidrule(lr){3-6}\cmidrule(lr){7-10}
      & & \textbf{PSNR} $\uparrow$ &  \textbf{SSIM} $\uparrow$  & \textbf{LPIPS} $\downarrow$ & \textbf{FID} $\downarrow$
        & \textbf{PSNR} $\uparrow$ &  \textbf{SSIM} $\uparrow$ & \textbf{LPIPS} $\downarrow$ & \textbf{FID} $\downarrow$ \\ \midrule
    
    \multirow{9}{*}{SR $4\times$} 
    & DPS & $25.86$ & $0.753$ & $0.269$ & $81.07$ & $21.13$ & $0.489$ & $0.361$ & $106.32$\\
    & DDRM & $26.58$ & $0.782$  & $0.282$ & $79.25$ & $22.62$ & $0.521$ & $0.324$ & $103.85$\\
    & DDNM & $28.03$ & $0.795$ & $0.197$ & $64.62$ & $23.96$ & $0.604$ & $0.475$ & $98.62$ \\
    & DCDP & $28.66$ & $0.807$ & $0.178$ & $53.81$ & -- & -- & -- & -- \\
    & FPS-SMC & $28.42$ & $0.813$ & $0.204$ & $49.25$ & $24.82$ & $0.703$ & $0.313$ & $97.51$ \\
    & DiffPIR & $26.64$ & -- & $0.260$ & $65.77$ & $23.18$ & -- & $0.371$ & $106.32$ \\
    & RED-Diff & $28.63$ & $0.748$ & $0.288$ & $126.78$ & $25.43$ & $0.639$ & $0.336$ & $153.37$ \\
    & DAPS  & $29.07$ & $0.818$ & $0.177$ & $51.44$ & $25.89$ & $0.694$ & $0.276$ & $83.57$ \\
    & MAS & $\mathbf{30.56}$ & $\mathbf{0.865}$ & $\mathbf{0.131}$ & $61.38$ & $\mathbf{27.20}$ & $\mathbf{0.751}$ & $\mathbf{0.215}$ & $88.61$ \\
    \midrule
    
    \multirow{8}{*}{Inpaint (Box)} 
    & DPS & $22.51$ & $0.792$ & $0.209$ & $61.27$ & $18.94$ & $0.722$ & $0.257$ & $126.52$\\
    & DDRM & $22.26$ & $0.801$ & $0.207$ & $78.62$ & $18.63$ & $0.733$ & $0.254$ & $116.37$ \\
    & DDNM & $24.47$ & $0.837$ & $0.235$ & $46.59$ & $21.64$ & $0.748$ & $0.319$ & $103.97$\\
    & DCDP & $23.89$ & $0.760$ & $0.163$ & $45.23$ & -- & -- & -- & --\\
    & FPS-SMC & $24.86$ & $0.823$ & $0.146$ & $48.34$ & ${22.16}$ & $0.726$ & $0.208$ & $111.58$ \\
    & RED-Diff & $24.68$ & $0.767$ & $0.175$ & $86.78$ & $21.32$ & $0.728$ & $0.247$ & $123.55$ \\
    & DAPS  & $24.07$ & $0.814$ & $0.133$ & $43.10$  & $21.43$ & $0.725$ & $0.214$ & $109.85$ \\
    & MAS & $\mathbf{24.95}$ & $\mathbf{0.879}$& $\mathbf{0.082}$ & $\mathbf{37.67}$ & $21.15$ & $\mathbf{0.817}$ & $\mathbf{0.168}$ & $\mathbf{95.96}$ \\
    \midrule
    
    \multirow{7}{*}{Inpaint (Random)} 
    & DPS & $25.46$ & $0.823$ & $0.203$ & $69.20$ & $23.52$ & $0.745$ & $0.297$ & $87.53$ \\
    & DDNM & $29.91$ & $0.817$ & $0.121$ & $44.37$ & ${31.16}$ & $0.841$ & $0.191$ & $63.84$ \\
    & DCDP & $30.69$ & $0.842$ & $0.142$ & $52.51$ & -- & -- & -- & -- \\
    & FPS-SMC & $28.21$ & $0.823$ & $0.261$ & $61.23$ & $24.52$ & $0.701$ & $0.316$ & $79.12$ \\
    & RED-Diff & $29.73$ & $0.814$ & $0.200$ & $104.19$ & $27.04$ & $0.753$ & $0.226$ & $92.24$ \\
    & DAPS  & $31.12$ & $0.844$ & $0.098$ & $32.17$  & $28.44$ & $0.775$ & $0.135$ & $54.25$\\
    & MAS & $\mathbf{33.10}$ & $\mathbf{0.923}$ & $\mathbf{0.073}$ & ${34.75}$ & $\mathbf{29.05}$ & $\mathbf{0.838}$ & $\mathbf{0.113}$ & $\mathbf{30.19}$ \\
    
    \midrule
    \multirow{3}{*}{Deblurring (Uniform)} 
    & DDNM & $26.58$ & $0.704$ & $0.210$ & $68.83$ & $25.69$ & $0.630$ & $0.261$ & $83.63$ \\
    & DDRM & $29.19$ & $0.835$ & $0.172$ & $87.12$ & $26.31$ & $0.711$ & $0.267$ & $118.36$ \\ 
    & DAPS  & $28.92$ & $0.758$ & $0.204$ & $76.57$ & $25.43$ & $0.616$ & $0.293$ & $103.55$  \\
    & MAS & $\mathbf{30.58}$ & $\mathbf{0.857}$ & $0.174$ & $103.88$ & $26.25$ & $0.700$ & $0.295$ & $141.58$ \\
    \midrule
     \multirow{5}{*}{Color} 
    & DDNM & $24.83$ & $0.868$ & $0.244$ & $85.15$ & $22.57$ & $0.884$ & $0.271$ & $87.48$ \\
    & DDRM & $23.27$ & $0.881$ & $0.250$ & $100.48$ & $21.12$ & $0.819$ & $0.346$ & $103.39$ \\ 
    & RED-Diff & $24.21$ & $0.785$ & $0.304$ & $107.64$ & $22.18$ & $0.782$ & $0.368$ & $104.40$ \\
    & DAPS  & $23.92$ & $0.825$ & $0.263$ & $88.09$ & $22.13$ & $0.830$ & $0.323$ & $89.30$ \\
    & MAS & $24.23$ & $\mathbf{0.919}$ & $\mathbf{0.187}$ & $\mathbf{72.33}$ & $\mathbf{22.66}$ & $\mathbf{0.886}$ & $\mathbf{0.258}$ & $\mathbf{83.17}$ \\
    \midrule
    \bottomrule
  \end{tabular}}
\end{table*}

\subsection{Image restoration}

\textbf{Inverse problems with Gaussian noise (known variance).} Quantitative results for inverse problems with Gaussian noise of known variance are shown in \cref{tab:noisy_gaussian}. MAS consistently demonstrates superior performance across most tasks, notably achieving significantly higher PSNRs. The table summarizes 5 tasks, 4 restoration quality metrics, and 2 datasets, resulting in a total of 40 evaluations. MAS demonstrates superior performance in 29 out of the 40 cases. Notably, MAS achieves improvements of more than 1 dB in 5 out of 10 instances.

\begin{table*}[t]
  \centering
  \caption{Quantitative evaluation of solving linear inverse problems with non-Gaussian noise (unknown strength).} 
  \label{tab:non_gaussian}
  \resizebox{0.9 \textwidth}{!}{%
  \begin{tabular}{@{}llcccccccc@{}}
    \toprule
    \multirow{2}{*}{\textbf{Task}}  & \multirow{2}{*}{\textbf{Method}} 
      & \multicolumn{4}{c}{\textbf{Salt peper noise}} & \multicolumn{4}{c}{\textbf{Periodic noise}} \\ \cmidrule(lr){3-6}\cmidrule(lr){7-10}
      & & \textbf{PSNR} $\uparrow$ & \textbf{SSIM} $\uparrow$  & \textbf{LPIPS} $\downarrow$ & \textbf{FID} $\downarrow$
        & \textbf{PSNR} $\uparrow$ & \textbf{SSIM} $\uparrow$ &  \textbf{LPIPS} $\downarrow$ & \textbf{FID} $\downarrow$ \\ \midrule
    \multirow{5}{*}{SR $8\times$}
      & DDNM & 13.02 & 0.289 & 0.710 & 377.54 & 18.61 & 0.492 & 0.495 & 268.36 \\
      & DDRM & 16.06 & 0.506 & 0.629 & 351.69 & 19.74 & 0.545 & 0.463 & 218.38 \\
      & $\Pi$GDM & 17.36 & 0.476 & 0.569 & 309.73 & 18.12 & 0.449 & 0.434 & 163.41 \\
      & RED-Diff & 14.21 & 0.357 & 0.668 & 342.09 & 19.47 & 0.596 & 0.416 & 224.07 \\
      & MAS (ours) & \textbf{20.05} & \textbf{0.605} & \textbf{0.390} & \textbf{129.80} & \textbf{20.10} & \textbf{0.591} & \textbf{0.395} & \textbf{137.57} \\
     \midrule
    \multirow{5}{*}{Inpaint (Box)}
      & DDNM & 15.55 & 0.248 & 0.533 & 247.99 & 18.60 & 0.621 & 0.341 & 147.80 \\
      & DDRM & 20.27 & 0.599 & 0.350 & 142.01 & 18.74 & 0.589 & 0.423 & 199.14 \\
      & $\Pi$GDM & 19.30 & 0.665 & 0.297 & 100.07 & 18.32 & 0.601 & 0.349 & 150.49 \\
      & RED-Diff & 15.75 & 0.287 & 0.523 & 255.85 & 19.13 & 0.638 & 0.338 & 159.99 \\
      & MAS (ours) & \textbf{22.78} & \textbf{0.723} & \textbf{0.244} & \textbf{90.15} & \textbf{19.13} & 0.581 & 0.407 & \textbf{138.56} \\
    \bottomrule
  \end{tabular}}
\end{table*}

\textbf{Inverse Problems with Non-Gaussian Noise (Unknown Strength).} Quantitative evaluations for linear inverse problems with unknown, non-Gaussian noise are presented in \cref{tab:non_gaussian}. MAS consistently outperforms baseline methods, highlighting the effectiveness of our approach in handling unknown noise conditions.

\begin{figure}
    \centering
    \includegraphics[width=.95\linewidth]{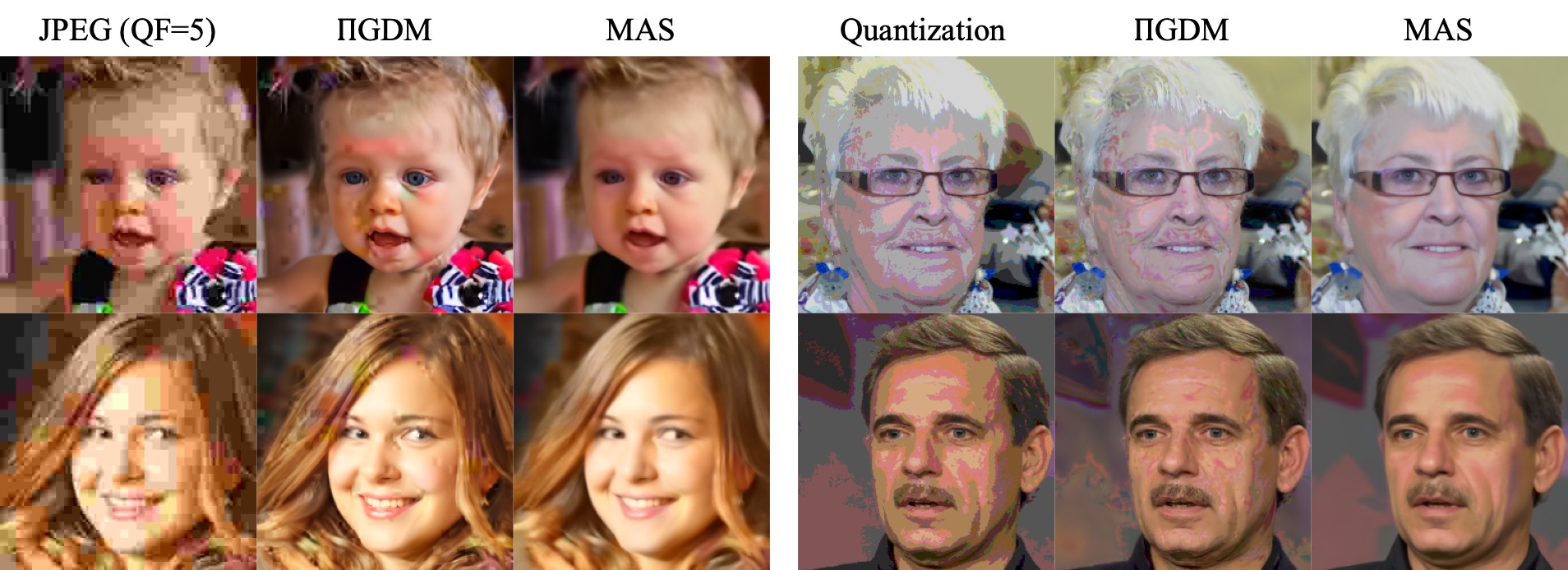}
    \caption{Results on JPEG (QF=5) and quantization restoration.}
    \label{fig:non_linear}
\end{figure}

\textbf{Inverse problems with non-differentiable measurements.} MAS is also capable of solving inverse problems with non-differentiable measurements, such as JPEG restoration and quantization. Results in \cref{tab:jpeg} and \cref{fig:non_linear} show that MAS achieves state-of-the-art performance without relying on the forward operator or knowledge of the degradation strength.

\textbf{Computational time analysis}. The computational efficiency of MAS is comparable to DDNM and substantially higher than DAPS. For example, on the SR task using the FFHQ-256 dataset with 200 diffusion steps, the non-parallel single-image sampling time for both DDNM and MAS is only 8 seconds per image, whereas DAPS requires 67 seconds (test were conducted on the same NVIDIA A6000 GPU).

\begin{table}[t]
    \centering
    \caption{
    Quantitative evaluation of  solving JPEG restoration and Quantization. We set $k = 1.0$ for $\mathrm{QF} = 5$ and $k = 3.0$ for $\mathrm{QF} = 2$, and $k=0.5$ for quantization. Both $\Pi$GDM and MAS use $\mathrm{NFE} = 20$, which yields the best performance (among $\mathrm{NFE} = 20$ and $\mathrm{NFE} = 100$). Notably, our method (MAS) does not require access to the forward operator or the strength of degration.}
    \label{tab:jpeg}
    \setlength{\tabcolsep}{4pt} 
    
    \resizebox{\textwidth}{!}{%
    \begin{tabular}{lcccccccccccc}
    \toprule
    \multirow{2}{*}{\textbf{Method}} 
        & \multicolumn{4}{c}{\textbf{JPEG Restoration (QF = 5)}} 
        & \multicolumn{4}{c}{\textbf{JPEG Restoration (QF = 2)}} 
        & \multicolumn{4}{c}{\textbf{Quantization (number of bits = 2)}} \\
    \cmidrule(lr){2-5} \cmidrule(lr){6-9} \cmidrule(lr){10-13}
        & PSNR $\uparrow$ &  SSIM $\uparrow$ & LPIPS $\downarrow$ & FID $\downarrow$
        & PSNR $\uparrow$ &  SSIM $\uparrow$ & LPIPS $\downarrow$ & FID $\downarrow$
        & PSNR $\uparrow$ &  SSIM $\uparrow$ & LPIPS $\downarrow$ & FID $\downarrow$ \\
    \midrule
    $\Pi$GDM           & $25.78$ & $0.750$ & $0.241$ & $89.82$ & $22.92$ & $0.653$ & $0.314$ & $112.27$ & $29.98$ & $0.823$ & $0.185$ & $124.57$ \\
    MAS (ours)         & $\mathbf{26.30}$ & $\mathbf{0.787}$ & $\mathbf{0.281}$ & $101.24$ & $\mathbf{23.72}$ & $\mathbf{0.772}$ & $\mathbf{0.335}$ & $114.85$ & $28.97$ & $\mathbf{0.837}$ & $\mathbf{0.196}$ & $\mathbf{69.61}$ \\
    \bottomrule
    \end{tabular}
    }
\end{table}

\section{Related work}

Diffusion models have also been successfully applied to linear inverse problems, including, compressed-sensing MRI (CS-MRI), and computed tomography (CT) \citep{kadkhodaie2021stochastic, song2020score, chung2022improving, kawar2022denoising, song2021solving}. They have also been extended to non-linear inverse problems such as Fourier phase retrieval, nonlinear deblurring, HDR, and JPEG restoration \citep{chung2022diffusion, song2023pseudoinverse, chung2023parallel, mardani2023variational}.

Methods to solve inverse problems include linear projection methods \citep{wang2022zero, kawar2022denoising, dou2024diffusion}, Monte Carlo sampling \citep{wu2023practical, phillips2024particle}, variational inference \citep{feng2023score, mardani2023variational, janati2024divide},  optimization-based approaches \citep{song2023solving, zhu2023denoising, li2024decoupled, wang2024dmplug, alkhouri2024sitcom, he2023manifold}, and Diffusion Posterior Sampling (DPS) \citep{zhang2024improving, chung2022diffusion, song2023loss, yu2023freedom, rout2024beyond, yang2024guidance, bansal2023universal, boys2023tweedie, song2023pseudoinverse, ho2022classifier}. Besides, InverseBench \citep{zheng2025inversebench} presents a benchmark for critical scientific applications, which present structural challenges that differ significantly from natural image restoration tasks. 




    

\section{Conclusion}
\label{sec:conclusion}
MAS improves upon existing methods by explicitly aligning the sampling process with measurement data, offering a broader optimization perspective that generalizes approaches like DDNM and DAPS. Beyond the noise-free case, MAS can be extended to: (1) known Gaussian noise, (2) unknown or non-Gaussian noise through adaptive parameterization, and (3) non-differentiable degradations (e.g., JPEG) by decoupling the forward operator from sampling. Extensive experiments show that MAS consistently outperforms state-of-the-art methods across a wide range of inverse problems. While MAS can handle non-differentiable measurements like JPEG restoration, it does not support general non-linear inverse problems, it's also promising to `calibrate' the noise introduced into $x_t$, such that maximizing the consistency to measurement.



\newpage

\bibliography{references}
\bibliographystyle{plain}

\newpage
\appendix

\section{Proofs}
\label{sec:proofs}

\subsection{Proof of \cref{prop:blr}.}
\begin{proof}
Let $x \equiv x_{\epsilon}$. The prior and likelihood are
\[
p(x \mid x_t) = \mathcal N(m_{0 \mid t},\,C_{0 \mid t}), 
\qquad
p(y \mid x) = \mathcal N(Hx,\,R),
\]
with $R = \sigma_y^2 I_m + \sigma_{\epsilon}^2 HH^\top$.
Denote $m \coloneqq m_{0 \mid t}$ and $C \coloneqq C_{0 \mid t}$.

The posterior is, up to normalization,
\[
p(x \mid x_t,y)\;\propto\; \exp\!\left(
-\tfrac12 (x-m)^\top C^{-1}(x-m)
-\tfrac12 (y-Hx)^\top R^{-1}(y-Hx)
\right).
\]

Expanding the exponent and collecting terms in $x$ gives
\[
\begin{aligned}
&-\tfrac12 \Big[
x^\top C^{-1} x - 2 x^\top C^{-1} m + m^\top C^{-1} m
+ x^\top H^\top R^{-1} H x - 2 x^\top H^\top R^{-1} y + y^\top R^{-1} y
\Big] \\
&= -\tfrac12 \Big[
x^\top (C^{-1}+H^\top R^{-1} H) x
- 2 x^\top (C^{-1} m + H^\top R^{-1} y)
\Big] + \text{(terms independent of $x$)}.
\end{aligned}
\]

This is the quadratic form of a Gaussian in $x$ with precision
\[
\Lambda \;=\; C^{-1} + H^\top R^{-1} H,
\]
and natural parameter
\[
\eta \;=\; C^{-1} m + H^\top R^{-1} y.
\]

Therefore the posterior is Gaussian $\mathcal N(\mu_{\text{post}},\,\Sigma_{\text{post}})$ with
\[
\Sigma_{\text{post}} = \Lambda^{-1}
= \big(C^{-1} + H^\top R^{-1} H\big)^{-1},\qquad
\mu_{\text{post}} = \Sigma_{\text{post}}\,\eta
= \big(C^{-1} + H^\top R^{-1} H\big)^{-1}\big(C^{-1} m + H^\top R^{-1} y\big).
\]

Restoring the original notation gives \eqref{eq:blr}. 
\end{proof}

\subsection{Proof of effieient linear solves in \cref{eq:eff_solves}}

\begin{lemma}\label{lem:svd_solves}
Let \(H\in\mathbb{R}^{m\times d}\) have (thin) singular–value decomposition
\(H = U\Sigma V^{\mathsf T}\) with orthogonal \(U\in\mathbb{R}^{m\times m}\),
\(V\in\mathbb{R}^{d\times d}\) and 
\(\Sigma=\operatorname{diag}(s_1,\dots,s_{r})\in\mathbb{R}^{m\times d}\),  
where \(r=\operatorname{rank}(H)\) and \(s_1\ge\cdots\ge s_r>0\).
For any scalars \(\eta_1\ge 0\) and \(\eta_2>0\) define
\[
   W \coloneqq \eta_1HH^{\mathsf T}+\eta_2\mathbb{I},
   \qquad
   Y \coloneqq \mathbb{I} + H^{\mathsf T}W^{-1}H.
\]
Then
\begin{equation}
\label{eq:W_and_Y_inverse}
  W^{-1}=U\,
      \operatorname{diag}\!\Bigl(\tfrac{1}{\eta_1s_i^{2}+\eta_2}\Bigr)_{i=1}^{m}
      U^{\mathsf T},
  \qquad
  Y^{-1}=V\,
      \operatorname{diag}\!\Bigl(
          \tfrac{1}{\,1+s_i^{2}/(\eta_1s_i^{2}+\eta_2)\,}
      \Bigr)_{i=1}^{d}
      V^{\mathsf T}.
\end{equation}
(When \(i>r\) we set \(s_i\!=\!0\).)
\end{lemma}

\begin{proof}
\textbf{(i) Inverting \(W\).}
Using the SVD,
\[
   W
   = \eta_1U\Sigma\Sigma^{\mathsf T}U^{\mathsf T}+\eta_2U\mathbb{I}U^{\mathsf T}
   = U\!\bigl(\eta_1\Sigma\Sigma^{\mathsf T}+\eta_2\mathbb{I}\bigr)U^{\mathsf T}.
\]
Because \(U\) is orthogonal, \(W^{-1}\) is obtained by inverting the
diagonal middle matrix:
\(
   (\eta_1\Sigma\Sigma^{\mathsf T}+\eta_2\mathbb{I})^{-1}
   =\operatorname{diag}\!\bigl(\frac{1}{\eta_1s_i^{2}+\eta_2}\bigr)_{i=1}^{m}.
\)
Substituting yields the first identity in \eqref{eq:W_and_Y_inverse}.

\smallskip
\noindent\textbf{(ii) Inverting \(Y\).}
Write
\[
   Y = \mathbb{I} + H^{\mathsf T}W^{-1}H
     = V\Sigma^{\mathsf T}U^{\mathsf T}
       \bigl[U\operatorname{diag}\!\bigl(\tfrac{1}{\eta_1s_i^{2}+\eta_2}\bigr)U^{\mathsf T}\bigr]
       U\Sigma V^{\mathsf T} + \mathbb{I},
\]
and simplify with \(U^{\mathsf T}U=\mathbb{I}\):
\[
   Y
   = V\!\Bigl[\Sigma^{\mathsf T}\!
       \operatorname{diag}\!\bigl(\tfrac{1}{\eta_1s_i^{2}+\eta_2}\bigr)\!
       \Sigma + \mathbb{I}\Bigr] V^{\mathsf T}.
\]
Because \(\Sigma^{\mathsf T}\!
       \operatorname{diag}\!\bigl(\tfrac{1}{\eta_1s_i^{2}+\eta_2}\bigr)
       \Sigma\) is diagonal with \(i^{\text{th}}\) entry
\(\tfrac{s_i^{2}}{\eta_1s_i^{2}+\eta_2}\),
the bracketed matrix is diagonal and hence trivial to invert, giving
the second identity in \eqref{eq:W_and_Y_inverse}.
\end{proof}

\subsection{Proof of \cref{eq:opt_per}}

\begin{proposition}\label{prop:pm_closed_form}
Let $H\in\mathbb{R}^{m\times d}$, $\eta_1\ge 0$ and $\eta_2>0$.
Define
\[
   W \;=\; \eta_1\,HH^{\mathsf T} + \eta_2\,\mathbb{I},
   \qquad
   Y \;=\; \mathbb{I} + H^{\mathsf T}W^{-1}H.
\]
For any $y\in\mathbb{R}^m$ and $m_{0 \mid t}\in\mathbb{R}^d$ consider the
strictly convex quadratic
\[
   \mathcal{L}(x_0) \;=\;
      \bigl\| x_0-m_{0 \mid t}\bigr\|_2^{2}
      \;+\;
      \bigl\|y-H\tilde x_0\bigr\|_{W^{-1}}^{2},
   \qquad
   \|v\|_{W^{-1}}^{2}=v^{\mathsf T}W^{-1}v.
\]
Its unique minimiser is
\begin{equation}\label{eq:x_pm}
   \tilde x_0^*
   \;=\;
   Y^{-1}\!\bigl[m_{0 \mid t}+H^{\mathsf T}W^{-1}y\bigr].
\end{equation}
\end{proposition}

\begin{proof}
Expand $\mathcal{L}$ and take its gradient:
\[
   \nabla_{\tilde x_0}\mathcal{L}
   \;=\;
   2\bigl(\tilde x_0-m_{0 \mid t}\bigr)
   \;-\;
   2H^{\mathsf T}W^{-1}\!\bigl(y-H\tilde x_0\bigr).
\]
Setting $\nabla_{\tilde x_0}\mathcal{L}=0$ gives the \emph{normal equation}
\[
   \bigl(\mathbb{I}+H^{\mathsf T}W^{-1}H\bigr)\tilde x_0
   \;=\;
   m_{0 \mid t}+H^{\mathsf T}W^{-1}y,
   \qquad\text{that is, } Y\,\tilde x_0 = m_{0 \mid t}+H^{\mathsf T}W^{-1}y.
\]
Because $\eta_2>0$ implies $W\succ 0$, we have $W^{-1}\succ 0$ and hence
$Y=\mathbb{I}+H^{\mathsf T}W^{-1}H\succ 0$; thus $Y$ is invertible and
\eqref{eq:x_pm} follows.

Finally, the Hessian of $\mathcal{L}$ is $2Y\succ 0$, so
$\mathcal{L}$ is strictly convex and the stationary point \eqref{eq:x_pm}
is indeed its unique global minimiser.
\end{proof}

\subsection{Proof of \cref{remark:con_ddnm}.}

\begin{proof}

As $\eta_2 = 0$,

\begin{equation}
    x_0^* = \left(\eta\, \mathbb{I} + H^{\dagger} H\right)^{-1} \left(\eta_1\, m_{0 \mid t} + H^{\dagger} y\right).
\end{equation}

To analyze the limit as \(\eta_1 \to 0\), decompose the space into two orthogonal components:
\begin{itemize}
    \item The range (or row space) of \(H\), on which \(H^{\dagger} H\) acts as the identity.
    \item Its nullspace, on which \(H^{\dagger} H\) is zero.
\end{itemize}

Let
\begin{equation}
    P = H^{\dagger} H,
\end{equation}
which is the orthogonal projection onto the row space of \(H\). Then any vector \(v\) can be decomposed as
\begin{equation}
    v = P v + (I-P)v.
\end{equation}
Notice that \(H^{\dagger}y\) lies in the row space (i.e. \(P\,H^{\dagger}y = H^{\dagger}y\)) and that \(m_{0 \mid t}\) can be decomposed as
\begin{equation}
    m_{0 \mid t} = Pm_{0 \mid t} + (I-P)m_{0 \mid t}.
\end{equation}

Since the eigenvalues of \(P\) are 0 and 1, the matrix \(\eta_1\, I + P\) has eigenvalues \(\eta_1\) (on the nullspace of \(P\)) and \(1+\eta_1\) (on the row space). Hence, its inverse acts as:
\begin{itemize}
    \item Multiplication by \(1/\eta_1\) on the nullspace,
    \item Multiplication by \(1/(1+\eta_1)\) on the row space.
\end{itemize}

Thus, we have
\begin{equation}
    \left(\eta_1\, I + P\right)^{-1}\left(H^{\dagger} y + \eta_1\, m_{0 \mid t}\right)
    = \frac{1}{1+\epsilon}\bigl(H^{\dagger}y + \eta_1\, Pm_{0 \mid t}\bigr) + \frac{1}{\eta_1}\bigl(\eta_1\,(I-P)m_{0 \mid t}\bigr).
\end{equation}
Simplify this to obtain
\begin{equation}
    \frac{1}{1+\eta_1}\, H^{\dagger}y + \frac{\eta_1}{1+\eta_1}\, Pm_{0 \mid t} + (I-P)m_{0 \mid t}.
\end{equation}
Now, taking the limit as \(\eta_1\to 0\):
\begin{itemize}
    \item \(\frac{1}{1+\eta_1}\to 1\),
    \item \(\frac{\eta_1}{1+\eta_1}\to 0\).
\end{itemize}

Therefore, the limit becomes
\begin{equation}
    \lim_{\eta_1\to 0}x_0^* = H^{\dagger}y + (I-P)m_{0 \mid t}.
\end{equation}
Recalling that \(P = H^{\dagger} H\), we rewrite this as
\begin{equation}
    H^{\dagger}y + m_{0 \mid t} - H^{\dagger} H\,m_{0 \mid t} = m_{0 \mid t} + H^{\dagger}\bigl(y - Hm_{0 \mid t}\bigr).
\end{equation}

Thus, in the limit where \(\eta_1\to 0\), we indeed have
\begin{equation}
    x_0^* = m_{0 \mid t} + H^{\dagger}\bigl(y - Hm_{0 \mid t}\bigr).
\end{equation}

This shows that, as the relative measurement noise \(\epsilon\) becomes much smaller compared to the prior uncertainty \(r_t\), the posterior expectation is the projection of \(\hat{x}_0^\theta\) onto the subspace \(\{ x: Hx = y \}\).
\end{proof}

\section{Additional method details}
\label{sec:sm}



\subsection{Addressing Gaussian noise}

Consider noisy image restoration problems in the form of $y = Hx + 
\color{blue}{\epsilon_y}$, where $\epsilon_y$ is the added noise. Then the measurement $y$ can be decomposed to the sum of clean measurement $y^{\text{clean}}\coloneqq Hx$ and measurement noise $\epsilon_y$. Calculating $x_0^*$ leads to:

\begin{align}
    x_0^* &= Y^{-1}[m_{0 \mid t} + H^{\mathsf T} W^{-1}y] \\
    &=  m_{0 \mid t} + (Y^{-1} - \mathbb{I}) m_{0 \mid t} + Y^{-1} H^TW^{-1} y 
\end{align}

where $\textcolor{blue}{Y^{-1}H^TW^{-1}\epsilon_y}$ is the extra noise introduced into $x_0^*$ and will be further introduced into $x_{t-\Delta t}$. 




To address Gaussian noise with known variance, we modify \cref{eq:pm} and \cref{eq:reverse_iter} as:

\begin{equation}
x_0^* = m_{0 \mid t} + \textcolor{blue}{\Sigma_t} [(Y^{-1} - \mathbb{I} ) m_{0 \mid t} + H^{\mathsf T} W^{-1}y]
\end{equation}

\begin{equation}
    x_{t - \Delta t} \sim \mathcal{N} (a_t \tilde{x}_0^{\text{pe}}(t, x, y) + b_t x_t, \textcolor{blue}{\Phi}_t \mathbb{I})
\end{equation}

Then $x_0^*$ is:

\begin{align}
x_0^* &= m_{0 \mid t} + \textcolor{blue}{\Sigma_t} [(Y^{-1} - \mathbb{I} ) m_{0 \mid t} + H^{\mathsf T} W^{-1}y] \\
    \\
    &= \underbrace{ m_{0 \mid t} + \textcolor{blue}{\Sigma_t} (Y^{-1} - \mathbb{I}) m_{0 \mid t} + Y^{-1} H^TW^{-1} y^{\text{clean}} }_{\coloneqq\; \tilde{x}_0^{\text{clean}}} + \textcolor{blue}{\Sigma_tY^{-1}H^{\mathsf T}W^{-1}\epsilon_y}
\end{align}

Then the iteration of the sampling process is:

\begin{align}
    x_{t - \Delta t} &= a_t x_0^*(t, x, y) + b_t x_t  + \epsilon_{\text{new}} , \quad \epsilon_{\text{new}} \sim \mathcal{N} (0, \Phi_t)\\
    &= a_t \tilde{x}_0^{\text{clean}} + b_t x_t + \underbrace{\textcolor{blue}{a_t \sigma_y  Y^{-1} H^{\mathsf T}W^{-1} \epsilon_y}}_{ \coloneqq \epsilon_{\text{intro}}}  + \epsilon_{\text{new}}
\end{align}

Suppose $\Sigma_t = V\mathrm{diag} \{ \lambda_{t1}, \cdots, \lambda_{td} \} V^T$ $\Phi_t = V \mathrm{diag} \{ \gamma_{t1}, \cdots \gamma_{td} \} V^T$. Then the introduced noise $\epsilon_{\text{intro}} = \textcolor{blue}{ a_t \sigma_y  Y^{-1} H^{\mathsf T}W^{-1} \epsilon_y}$ is still a Gaussian distribution: $\epsilon_{\text{intro}} \sim \mathcal{N} (0,  V D_t V^T )$, with $D_t = diag \{d_{t1}, \cdots, d_{td} \}$:
\begin{equation}
    d_{ti} \;=\;
    \begin{cases}
    \displaystyle
\frac{a_t^{2}\,\sigma_y^{2}\, s_i^{2} \lambda_{ti}^2}{\bigl[(\eta_{1}+1)\,s_i^{2}+\eta_{2}\bigr]^{2}}, 
    & s_i \neq 0,\\[8pt]
    0,
    & s_i = 0,
  \end{cases}
\end{equation}

The choice of and $\Phi_t$ need to ensure the total noise injected to $x_{t - \Delta t}$ conforms the iteration in \cref{eq:reverse_iter}. 

\begin{equation}
   \epsilon_{\text{new}} + \epsilon_{\text{intro}} \sim \mathcal{N} (0, c_t^2 \mathbb{I}) 
\end{equation}

To construct $\epsilon_{\text{new}}$, we define a new diagonal matrix $\Gamma_t( = diag \{ \gamma_{t1}, \cdots \gamma_{td} \})$:

\begin{equation}
    \gamma_{ti} \;=\;
    \begin{cases}
    \displaystyle
    c_t^2 - \frac{a_t^{2}\,\sigma_y^{2}\, s_i^{2} \lambda_{ti}^2}{\bigl[(\eta_{1}+1)\,s_i^{2}+\eta_{2}\bigr]^{2}},
    & s_i \neq 0,\\[8pt]
    c_t^2,
    & s_i = 0,
  \end{cases}
\end{equation}

Now we can yield $\epsilon_{\text{new}}$ by sampling from $\mathcal{N} (0, V\Gamma_tV^T)$ to ensure that $\epsilon_{\text{intro}} + \epsilon_{\text{new}} \sim \mathcal{N} (0, c_t^2 \mathbb{I})$. We need to make sure $\lambda_{ti}$ guarantees the noise level of the introduced noise does not exceed the pre-defined noise level $c_t$, we also hope $\lambda_{t_i}$ as close as $1$ as possible. Therefore, 

\begin{equation}
    \lambda_{ti} \;=\;
    \begin{cases}
    \displaystyle
    1, & c_t \geq \frac{a_t\,\sigma_y\, s_i }{(\eta_{1}+1)\,s_i^{2}+\eta_{2}} , \\
    \frac{c_t ((\eta+1) s_i^2 + \eta_2)}{a_t \sigma_y s_i}, 
    & c_t < \frac{a_t\,\sigma_y\, s_i }{(\eta_{1}+1)\,s_i^{2}+\eta_{2}} ,\\[8pt]
    1,
    & s_i = 0.
  \end{cases}
\end{equation}

In practice, we found that setting $\sigma_y$ slightly larger than the true $\sigma_y$ is beneficial, possibly because the denoiser is more sensitive to excessive noise.

\subsection{Efficient calculation via SVD decomposition}
\label{sec:eff_cal}

Let \(H=U\Sigma V^{\mathsf T}\) with singular values
\(s_1,\dots,s_n\).
Then  

\begin{equation}
W^{-1}
=U\,\mathrm{diag}\!\Bigl(\tfrac{1}{\eta_1s_i^{2}+\eta_2}\Bigr)U^{\mathsf T},
\qquad
Y^{-1}
  =V\,\mathrm{diag}\!\Bigl(
      \tfrac{1}{\,1+s_i^{2}/(\eta_1s_i^{2}+\eta_2)\,}
    \Bigr)V^{\mathsf T},
    \label{eq:eff_solves}
\end{equation}

see \cref{sec:proofs} for the proof. Hence both \(W^{-1}v\) and \(Y^{-1}u\) reduce to inexpensive diagonal scalings in the SVD basis, avoiding the calculation of any explicit matrix inversion or square-root. The algorithm of MAS for inverse problem is provided in \cref{algo:mas}.

As $\eta_1 < 0$, $W$ could be non-invertible. However, $W=U \text{diag}(\eta_1 s_1^2, \cdots, \eta_1 s_r^2, \eta_2,\cdots, \eta_2)U^T$. Hence $W$ is invertible if $\eta_1s_i^2+\eta_2 \neq 0$ for every $i$. Even when $\eta_2 = 0$ and $\eta_1 < 0$ make $W$ singular, the update $W^{\dagger}y$ uses the Moore-Penrose pseudo-inverse $W^{\dagger}$, which is always well-defined. The pseudo‑inverse acts like an ordinary inverse on the range of $H$ and leaves the null‑space untouched, so the sampler remains stable. Empirically, small negative values ($-0.5< \eta_1 < 0$) often give the visual boost without instability, as demonstrated in the ablation studies in \cref{sec:experiments}

\subsection{Addressing unknown noise and non-differentiable measurements}

As the measurement noise $\epsilon_y$ is non-Gaussian or unknown, it's difficult to ensure the total noise introduced in $x_{t-\Delta t}$ is $\mathcal{N} (0, c_t^2 \mathbb{I})$. In this case, we calculate $x_0^*$ using \cref{eq:pm} and update $x_{t-\Delta t}$ using \cref{eq:reverse_iter}. Then $x_0^*$ is:

\begin{align}
x_0^* &= m_{0 \mid t} + [(Y^{-1} - \mathbb{I} ) m_{0 \mid t} + H^{\mathsf T} W^{-1}y] \\
    \\
    &= \underbrace{ m_{0 \mid t} +  (Y^{-1} - \mathbb{I}) m_{0 \mid t} + Y^{-1} H^TW^{-1} y^{\text{clean}} }_{\coloneqq\; \tilde{x}_0^{\text{clean}}} + \textcolor{blue}{Y^{-1}H^{\mathsf T}W^{-1}\epsilon_y}
\end{align}

where $\textcolor{blue}{Y^{-1}H^TW^{-1}\epsilon_y}$ is the extra noise introduced into $x_0^*$ and will be further introduced into $x_{t-\Delta t}$:

\begin{equation}
    \begin{aligned}
    x_{t-\Delta t} &= a_t x_0^* + b_t x_t + \epsilon_{\text{new}}, \\
    &= a_t \tilde{x}_0^{\text{clean}} + b_t x_t 
    + \underbrace{\textcolor{blue}{a_t Y^{-1}H^{\mathsf T}W^{-1}\epsilon_y}}_{\coloneqq\; \epsilon_{\text{intro}}}
    + \epsilon_{\text{new}},
    \end{aligned}
\end{equation}

where $\epsilon_{\text{new}}$ the noise added by diffusion process, which should be specifically designed to ensure $x_{t-\Delta t}$ is sampled from correct distribution as in \cref{eq:reverse_iter}, i.e., the total noise $\epsilon_{\text{intro}} + \epsilon_{\text{new}} \sim \mathcal{N} (0, c_t^2 \mathbb{I})$. However, as $\epsilon_y$ is unknown noise, we have no information about the introduced noise $\epsilon_{\text{intro}}$. To solve this problem, we made the following principles: (i) despite that fact that we cannot guarantee $\epsilon_{\text{intro}} + \epsilon_{\text{new}} \sim \mathcal{N} (0, c_t^2 \mathbb{I})$, we still hope $\epsilon_{\text{intro}} + \epsilon_{\text{new}}$ is as close to $\mathcal{N} (0, c_t^2 \mathbb{I})$ as possible; (ii) small $\eta_1$ and $\eta_2$ are helpful to maximize the alignment to measurement $y$. Notably, 

\begin{align}
    \epsilon_{\text{intro}} &= a_t  Y^{-1}H^{\mathsf T}W^{-1} \epsilon_y \\
    &= a_t  V \mathrm{diag}\left(\frac{s_i}{(\eta_1 + 1)s_i^2 + \eta_2}\right) U^T \epsilon_y
    \label{eq:ep_intro_2}
\end{align}

$\eta_1$ and $\eta_2$ are two variables that control the noise level of $\epsilon_{\text{intro}}$. In the implementation, we still sample $\epsilon_{\text{new}}$ from Gaussian distribution $\mathcal{N} (0, c_t^2 \mathbb{I})$. Then the problem becomes how to select $\eta_1$ and $\eta_2$ to meet the above 2 principles. For common image restoration tasks like SR, Deblurring, inpainting, Colorization, The maximum eigenvalue value $s_{\max} = \max \{ s_i\} <=1$. Therefore, adjusting $\eta_2$ is more likely to reduce the strength of $\epsilon_{\text{intro}}$.








Fix a tiny baseline $\eta_1 \in [-0.4, 0) \cup (0, 0.1)$, for example, we can choose $\eta_2(t)=k a_t/c_t$, where $k$ is a constant wihch depends on the measurement noise $\epsilon_y$. This meets Principle (i) by bounding the extra noise, and Principle (ii) by keeping both $\eta_1$, $\eta_2$ small enough to preserve measurement alignment.

\section{Limitations}
\label{sec:limitations}

While MAS can, in principle, be generalized to nonlinear inverse problems, explicitly formulating the likelihood term $p(y \mid x_{\epsilon})$ becomes challenging. Developing effective sampling techniques under this setting is a promising direction for future research.

\section{Impact Statement}
\label{sec:impact_statement}
Our method can improve image restoration under challenging noise and degradation conditions, which may benefit applications in medical imaging, scientific visualization, cultural heritage preservation, and general photography. However, it is important to note that as with many generative and restoration models, our method could be misused for malicious image manipulation.

\section{Pytorch-like code implementation}
\label{sec:pytorch_code}
Here we provide a basic PyTorch-Like implementation of the calculation of $x_0^*$ in \cref{eq:pm}, shown in \cref{lst:mas_code}.

\begin{listing}[h]
\begin{minted}[
    linenos,
    fontsize=\small,
    frame=lines,
    baselinestretch=1.0,
    breaklines,
]{python}
@torch.no_grad()
def mas(
    H, x0_hat, y,
    eta_1=-0.2, eta_2=0.0
):
    bs, _, H_img, W_img = x0_hat.shape
    x0_hat = x0_hat.view(bs,  -1)
    y        = y.view( bs, -1)                   # measurement dim m
    ut_y       = H.Ut(y)                         # (bs, m)
    singulars  = H.singulars()                   # (m,)
    nz         = singulars > 0                  # boolean mask
    scale1     = 1.0 / (singulars[nz] ** 2 * eta_1 + eta_2)
    ut_y[:, nz] = ut_y[:, nz] * scale1           # broadcasting OK
    u_y        = H.U(ut_y)                       # (bs, m)
    rhs        = x0_hat + H.Ht(u_y)             # (bs, d)
    
    vt_rhs     = H.Vt(rhs)                    # (bs, d)
    scale2     = 1.0 / (1.0 + singulars[nz]**2 / (singulars[nz]**2 * eta_1 + eta_2))
    vt_rhs[:, nz] = vt_rhs[:, nz] * scale2
    x0_pm    = H.V(vt_rhs)                     # (bs, d)
    x0_pm = x0_pm.view(bs, 3, H_img, W_img)
    return x0_pm
\end{minted}
\caption{PyTorch-like implementation of the calculation of $x_0^*$ in \cref{eq:pm}.}
\label{lst:mas_code}
\end{listing}

\section{Experimental details}
\label{sec:exp_details}
\subsection{Details of the degradation operators}

\textbf{Super-resolution.} We use the downsampler with bicubic kernel as the forward operator.

\textbf{Deblurring.} For deblurring experiments, We use uniform blur kernel to to implement blurring operation.


\textbf{Inpaint (Random).} Random Inpainting uses a generated random mask where each pixel has a 70\% chance of being masked, following the settings in \citep{song2023solving}.


\textbf{Inpaint (box).} We use a fixed square mask of size $128 \times 128$ pixels placed at the center of the image. 

\textbf{Colorization.} We simulate grayscale degradation by applying a fixed linear transformation to each pixel using the matrix $[\frac{1}{3}, \frac{1}{3}, \frac{1}{3}]$, replacing each RGB pixel with its average intensity.

\subsection{Details of the baseline models}

\textbf{Sampler}. Most experiments on diffusion models leverage DDIM \citep{song2020denoising} sampling.

\textbf{DDRM} \citep{kawar2022denoising}. $\eta_B=1.0$, $\eta=0.85$ with DDIM sampler, as advised in the original paper. 

\textbf{$\Pi$GDM} \citep{song2023pseudoinverse}. $\eta=1.0$, with DDIM sampler, as advised in the original paper.

\textbf{Reddiff} \citep{mardani2023variational}. $\lambda=0.25$, with DDIM sampler, as advised in the original paper.  

\textbf{DDNM} \citep{wang2022zero}. $\eta=0.85$, with DDIM sampler, as advised in the original paper.  

\textbf{DAPS} \citep{zhang2024improving}.  $\tau=0.01$, with EDM sampler, as advised in the original paper.

\section{Additional results}
\label{sec:add_results}

\subsection{Computational time}
\label{sec:com_res}

The computational time of MAS on solving inverse problems is shown in \cref{tab:time}. Our model achieves similar efficiency to DDNM and DDRM, demonstrating that MAS introduces minimal overhead while maintaining competitive runtime performance.

\begin{table}[t]
    \centering
    \small
    \caption{Sampling time (Sec) per image of MAS on deblurring and super-resolution with FFHQ 256, evaluated using a single NVIDIA A6000 48G GPU. We set NFE=20 and batch size = 20 for all of the methods. }
    \begin{tabular}{lccccc}
    \toprule
    Method & MAS & $\Pi$GDM & DDNM & DDRM & RED-Diff \\
    \midrule
    Deblurring & 0.128 & 0.278 & 0.127 & 0.127 & 0.119 \\
    SR ($8\times$) & 0.131 & 0.282 & 0.131 &0.131 & 0.125\\
    \bottomrule
    \end{tabular}
    \label{tab:time}
\end{table}

\begin{figure}[h]
    \centering
\includegraphics[width=1.\linewidth]{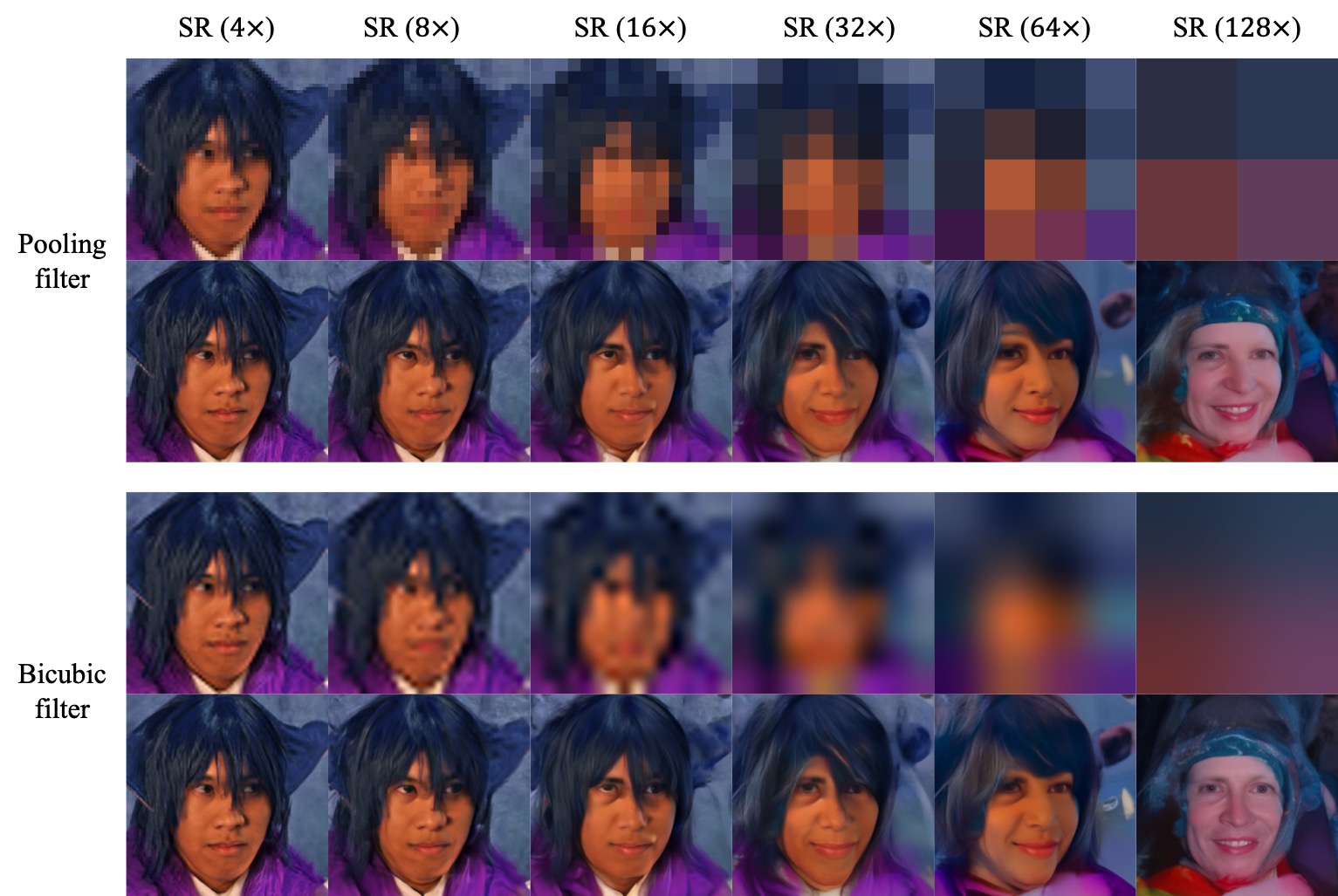}
    \caption{Super-resolution restoration over various strength of degradation. We set $\eta_1 = -0.4$ and $\eta_2 = 0$ for all tasks. For sampling process, we set $\eta=0.6$.}
    \label{fig:sr}
\end{figure}

\subsection{Adding Gaussian noise}

We present the results of solving $8\times$ super-resolution under varying levels of Gaussian noise in \cref{fig:sr_g_noise}. The visualizations demonstrate that MAS maintains strong restoration performance, even under high noise conditions.

\begin{figure}[t]
    \centering
\includegraphics[width=0.95 \linewidth]{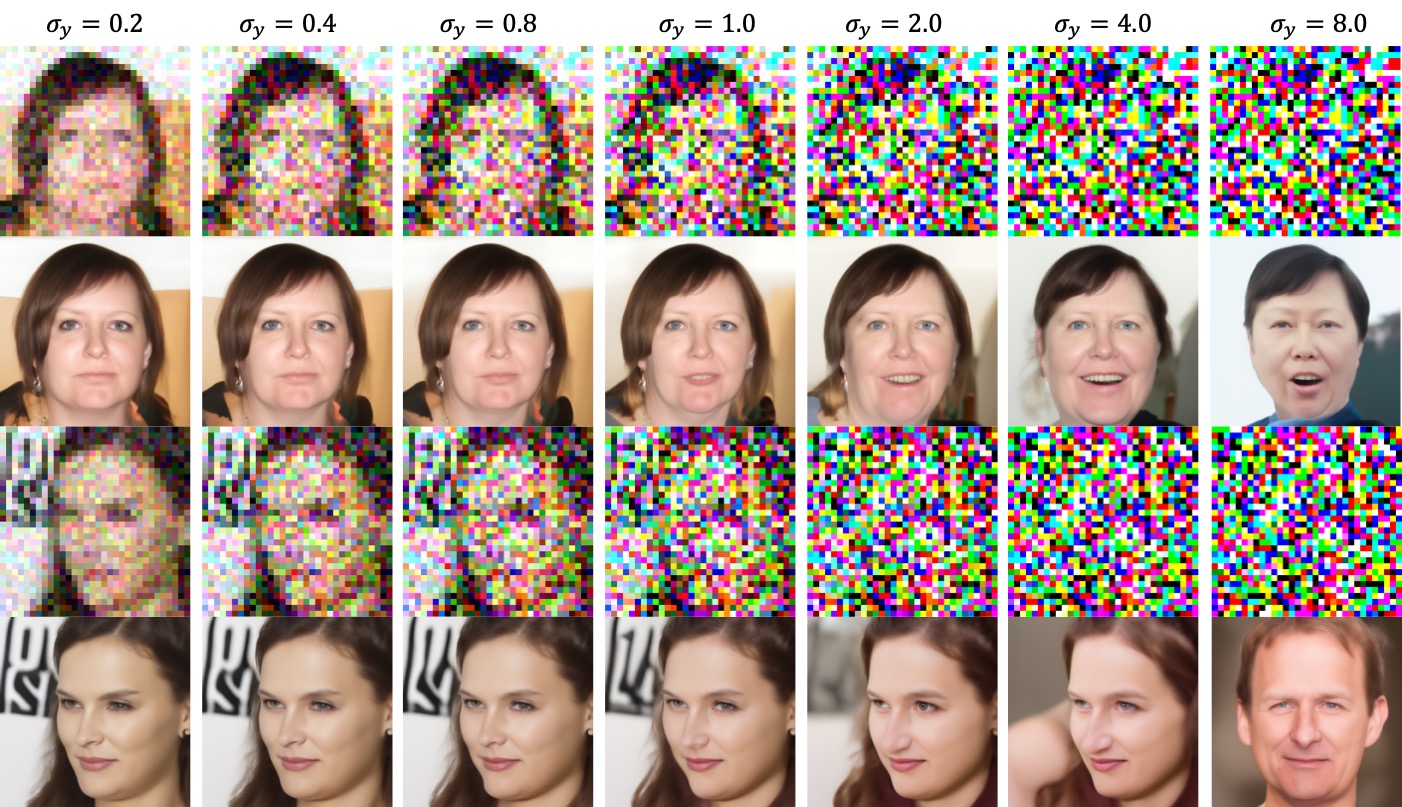}
    \caption{MAS for solving super-resolution ($8 \times$) with various strength of Gaussian noise.}
    
    \label{fig:sr_g_noise}
\end{figure}

\subsection{Adding non-differentiable degragation}

We present the results of solving $8\times$ super-resolution with non-differentiable degradations, including JPEG compression and quantization, in \cref{fig:sr_jpeg}.

\begin{figure}[t]
    \centering
    \includegraphics[width=0.95\linewidth]{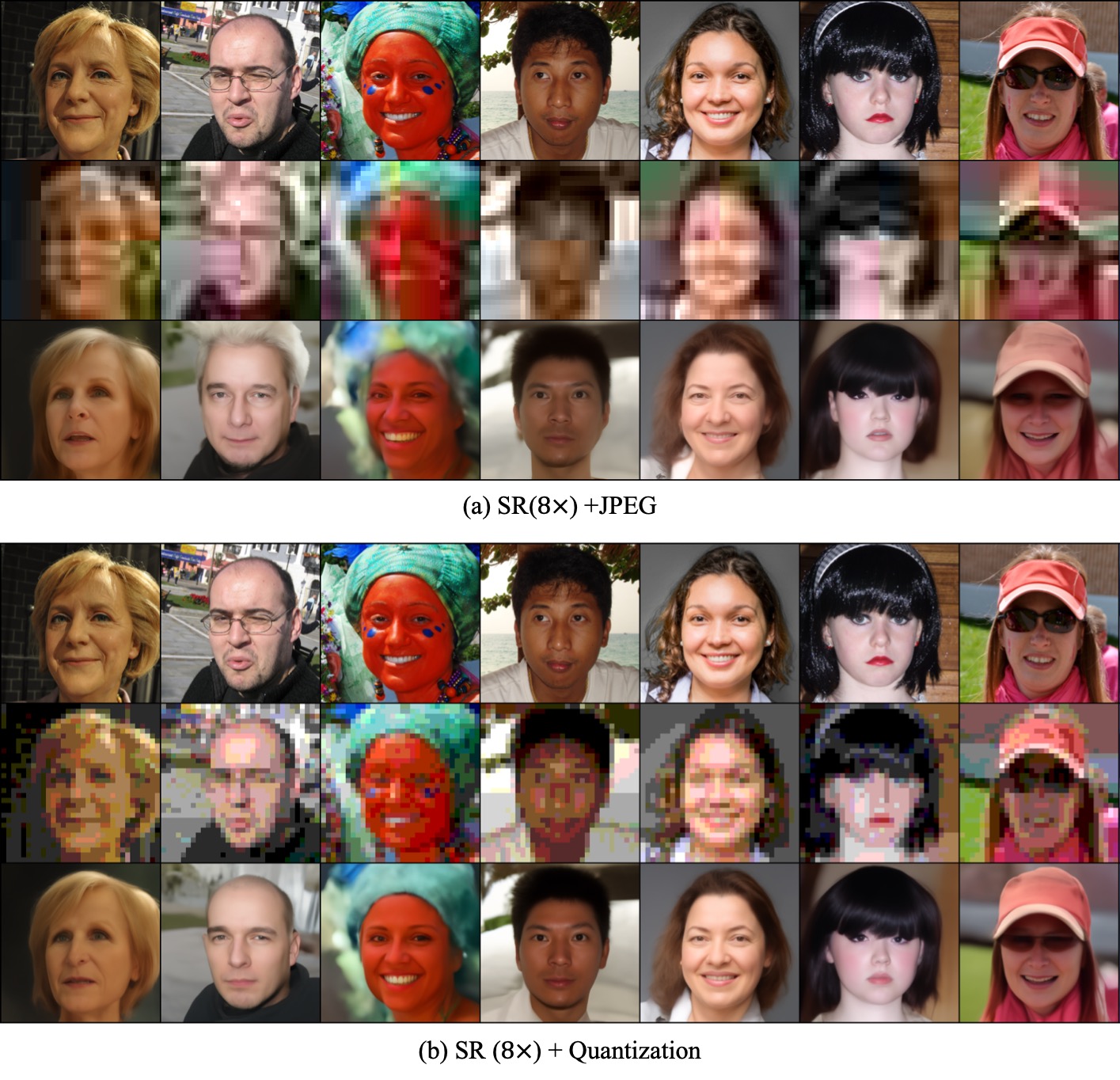}
    \caption{Additional visualization of super-resolution with unknown noise.}
    \label{fig:sr_jpeg}
\end{figure}


\section{Licenses}
\label{sec:licenses}
\textbf{FFHQ Dataset.}  
We use the Flickr-Faces-HQ (FFHQ) dataset released by NVIDIA under the Creative Commons BY-NC-SA 4.0 license. The dataset is intended for non-commercial research purposes only. More details are available at: \url{https://github.com/NVlabs/ffhq-dataset}.

\textbf{ImageNet Dataset.}  
The ImageNet dataset is used under the terms of its academic research license. Access requires agreement to ImageNet's data use policy, and redistribution is not permitted. More information is available at: \url{https://image-net.org/download}.

\end{document}

%% file: math_commands.tex
\usepackage{float}
\usepackage{amsmath,amsfonts,bm}

\def\eqref#1{equation~\ref{#1}}

\def\1{\bm{1}}

\DeclareMathAlphabet{\mathsfit}{\encodingdefault}{\sfdefault}{m}{sl}
\SetMathAlphabet{\mathsfit}{bold}{\encodingdefault}{\sfdefault}{bx}{n}

